%% file: main.tex
\documentclass[runningheads,envcountsname,a4paper]{llncs}

\usepackage{algorithm}
\usepackage[noend]{algpseudocode}
\usepackage{amsmath}
\usepackage{amssymb}
\usepackage{amsfonts}
\usepackage{appendix}
\usepackage{booktabs}
\usepackage[shortlabels]{enumitem}
\usepackage{float}
\usepackage{graphicx}
\usepackage{tikz}
\usepackage{xcolor}
\usepackage[hyphens]{url}

\usetikzlibrary{3d}
\usetikzlibrary{backgrounds,scopes}

\input{notations.tex}

\begin{document}
\title{Self-Bounding Majority Vote Learning Algorithms by the Direct Minimization of a Tight PAC-Bayesian C-Bound}
\toctitle{Self-Bounding Majority Vote Learning Algorithms by the Direct Minimization of a Tight PAC-Bayesian C-Bound}
\titlerunning{Majority Vote Learning by Direct Minimization of PAC-Bayesian C-Bound}

\author{Paul Viallard\inst{1}
\and Pascal Germain\inst{2}
\and Amaury Habrard\inst{1}
\and Emilie Morvant\inst{1}
}
\tocauthor{Paul~Viallard, Pascal~Germain, Amaury~Habrard, Emilie~Morvant}
\authorrunning{P. Viallard, P. Germain, A. Habrard, E. Morvant}

\institute{Univ Lyon, UJM-Saint-Etienne, CNRS, Institut d Optique Graduate School,\protect\\ Laboratoire Hubert Curien UMR 5516, F-42023, SAINT-ETIENNE, France
\email{firstname.name@univ-st-etienne.fr}\\
\and
D\'epartement d'informatique et de g\'enie logiciel, Universit\'e Laval,	Qu\'ebec, Canada\\
\email{pascal.germain@ift.ulaval.ca}}

\maketitle
\setcounter{footnote}{0}

\begin{abstract}
In the PAC-Bayesian literature, the C-Bound refers to an insightful relation between the risk of a majority vote classifier (under the zero-one loss) and the first two moments of its margin (\textit{i.e.}, the expected margin and the voters' diversity).
Until now, learning algorithms developed in this framework minimize the empirical version of the C-Bound, instead of explicit PAC-Bayesian generalization bounds. 
In this paper, by directly optimizing PAC-Bayesian guarantees on the C-Bound, we derive self-bounding majority vote learning algorithms.
Moreover, our algorithms based on gradient descent are scalable and lead to accurate predictors paired with non-vacuous guarantees.

\end{abstract}

\section{Introduction}

In machine learning, ensemble methods~\cite{Dietterich2000} aim to combine hypotheses to make predictive models more robust and accurate.
A weighted majority vote learning procedure is an ensemble method for classification where each voter/hypothesis is assigned a weight ({\it i.e.}, its influence in the final voting).
Among the famous majority vote methods, we can cite Boosting~\cite{FreundSchapire1996}, Bagging~\cite{Breiman1996}, or Random Forest~\cite{Breiman2001}.
Interestingly, most of the kernel-based classifiers, like Support Vector Machines~\cite{BoserGuyonVapnik1992,CortesVapnik1995}, can be seen as a majority vote of kernel functions.
Understanding when and why weighted majority votes perform better than a single hypothesis is challenging.
To study the generalization abilities of such majority votes, the PAC-Bayesian framework~\cite{TaylorWilliamson1997,McAllester1999} offers powerful tools to obtain Probably Approximately Correct (PAC) generalization bounds.
Motivated by the fact that PAC-Bayesian analyses can lead to tight bounds~(\eg, \cite{HernandezAmbroladzeTaylorSun2012}), developing algorithms to minimize such bounds is an important direction~({\it e.g.},~\cite{GermainLacasseLavioletteMarchand2009,GermainLacasseLavioletteMarchandRoy2015,DziugaiteRoy2017,MasegosaLorenzenIgelSeldin2020}).

We focus on a class of PAC-Bayesian algorithms minimizing an upper bound on the majority vote's risk  called the C-Bound\footnote{The C-Bound was introduced by Breiman in the context of Random Forest~\cite{Breiman2001}.} in the PAC-Bayesian literature~\cite{LacasseLavioletteMarchandGermainUsunier2006}.
This bound has the advantage of involving the majority vote's margin and its second statistical moment, {\it i.e.}, the diversity of the voters. 
Indeed, these elements are important when one learns a  combination~\cite{Dietterich2000,Kuncheva2014}: A good majority vote is made up of voters that are ``accurate enough'' and ``sufficiently diverse''.
Various algorithms have been proposed to minimize the C-Bound:  \algor{MinCq}~\cite{RoyLavioletteMarchand2011}, \mbox{\algor{P-MinCq}}~\cite{BelletHabrardMorvantSebban2014}, \algor{CqBoost}~\cite{RoyMarchandLaviolette2016}, or \algor{CB-Boost}~\cite{BauvinCapponiRoyLaviolette2020}.
Despite being empirically efficient, and justified by theoretical analyses based on the C-Bound, all these methods minimize \emph{only} the empirical C-Bound and not directly a PAC-Bayesian generalization bound on the C-Bound.
This can lead to vacuous generalization bound values and thus to poor risk certificates.

In this paper, we cover three different PAC-Bayesian viewpoints on generalization bounds for the C-Bound~\cite{McAllester2003,Seeger2002,LacasseLavioletteMarchandGermainUsunier2006}.
Starting from these three views, we derive three algorithms to optimize generalization bounds on the C-Bound. 
By doing so, we achieve \emph{self-bounding algorithms}~\cite{Freund1998}: the predictor returned by the learner comes with a statistically valid risk upper bound. 
Importantly, our algorithms rely on fast gradient descent procedures. 
As far as we know, this is the first work that proposes both efficient algorithms for C-Bound optimization and non-trivial risk bound values. 

The paper is organized as follows.
Section~\ref{section:notations} introduces the setting. 
Section~\ref{section:pac-bayesian-C-Bounds} recalls the PAC-Bayes bounds on which we build our results. 
Our self-bounding algorithms leading to non-vacuous PAC-Bayesian bounds are described in Section~\ref{section:contribution}.
We provide experiments in Section~\ref{section:expe}, and  conclude in Section~\ref{section:conclusion}.

\section{Majority Vote Learning}
\label{section:notations}
\subsection{Notations and Setting} 
We stand in the context of learning a weighted majority vote for binary classification. 
Let $\Xcal\!\subseteq\! \R^{d}$ be a $d$-dimensional input space, and $\Ycal{=}\{-1, +1\}$ be the label space.
We assume an unknown data distribution $\Dcal$ on $\Xcal{\times}\Ycal$, we denote by $\Dcal_{\Xcal}$ the marginal distribution on $\Xcal$.
A learning algorithm is provided with a learning sample  $\Scal{=}\{(\xbf_i, y_i)\}_{i=1}^{m}$  where each example $(\xbf_i,y_i)$ is drawn {\it i.i.d.} from $\Dcal$, we denote by $\Scal{\sim}\Dcal^m$ the random draw of such a sample.
Given $\Hcal$ a hypothesis set constituted by so-called {\it voters} \mbox{$h:\Xcal{\rightarrow}\Ycal$}, and  $\Scal$, the learner aims to find a weighted combination of the voters \mbox{from $\Hcal$}; the weights are modeled by a distribution on $\Hcal$.
To learn such a combination in the PAC-Bayesian framework, we assume a {\it prior} distribution $\P$ on $\Hcal$, and---after the observation of $\Scal$---we learn a {\it posterior} distribution $\Q$ on $\Hcal$.
More precisely, we aim to learn a well-performing classifier that is expressed as a \mbox{$\Q$-\textit{weighted majority vote}} $\MVQ$ defined as  
\begin{align*}\forall \xbf\in \Xcal,\quad  &\MVQ(\xbf) \ \triangleq\ \sign\LP\underset{h\sim\Q}{\EE}h(\xbf)\RP\ =\ \sign\LP\,\sum_{h\in\Hcal} \Q(h) h(\xbf)\RP.
\end{align*}
We thus want to learn $\MVQ$ that  commits as few errors as possible on unseen data from $\Dcal$,
{\it i.e.}, that leads to a low true risk $r^{\MV}_{\Dcal}\!(\Q)$ under the \mbox{{\small 0-1}-loss defined as}
\begin{align*}
r^{\MV}_{\Dcal}\!(\Q)
    & \triangleq \EE_{(\xbf, y)\sim\D} \Ibf\Big[   \MVQ(\xbf) \ne y\Big], \quad \mbox{where} \ \Ibf[a] = \left\{\begin{array}{ll}1 & \text{if the assertion $a$ is true,}\\ 0 & \text{otherwise.}\end{array}\right.
\end{align*}

\subsection{Gibbs Risk, Joint Error and C-Bound}
\label{sec:gibbs}
Since $\Dcal$ is unknown, a common way to try to minimize $r^{\MV}_{\Dcal}\!(\Q)$  
is the minimization of its empirical counterpart $r^{\MV}_{\Scal}\!(\Q) = \frac{1}{m}\sum_{i=1}^m \Ibf\left[\MVQ(\xbf_i) {\ne} y_i\right]$ computed on the learning sample $\Scal$ through the Empirical Risk Minimization principle.
However, learning the weights by the direct minimization of $r^{\MV}_{\Scal}\!(\Q)$ does not necessarily lead to a low true risk.
One solution consists then in looking for precise estimators or generalization bounds of the true risk $r^{\MV}_{\Dcal}\!(\Q)$  to minimize them.
In the PAC-Bayesian theory, a well-known estimator of the true risk $r^{\MV}_{\Dcal}\!(\Q)$ is the \textbf{Gibbs risk} defined as the \mbox{$\Q$-average} risk of the voters as
$$r_{\Dcal}(\Q)\ =\ \EE_{h\sim \Q} \EE_{(\xbf, y)\sim\D}  \Ibf\left[h(\xbf) \ne y\right].$$ 
Its empirical counterpart is defined as $r_{\Scal}(\Q){=}\frac{1}{m}\sum_{i=1}^{m}\EE_{h\sim \Q}\Ibf\left[h(\xbf_i) \ne y_i\right]$.
However, in ensemble methods where one wants to combine voters efficiently, the Gibbs risk appears to be an unfair estimator since it does not take into account the fact that a combination of voters has to compensate for the individual errors.
This is highlighted by the decomposition of  $r_{\Dcal}(\Q)$ in Equation~\eqref{eq:risk-err-disa} (due to Lacasse {\it et al.}~\cite{LacasseLavioletteMarchandGermainUsunier2006}) into  the expected \textbf{disagreement} and the expected \textbf{joint error}, respectively defined by 
\begin{align*}
       d_{\Dcal}(\Q) &= \EE_{h_1\sim\Q}\EE_{h_2\sim\Q} \EE_{\xbf\sim\D_\Xcal}\Ibf\big[h_1(\xbf) \ne h_2(\xbf)\big],\\
       \mbox{and}\quad  e_{\Dcal}(\Q) &= \EE_{h_1\sim\Q}\EE_{h_2\sim\Q}
    \EE_{(\xbf, y)\sim\Dcal}\! \Ibf\big[h_1(\xbf) \ne y\big]\Ibf\big[h_2(\xbf) \ne y\big].
\end{align*}
Indeed, an increase of the voter's diversity, captured by the disagreement $d_{\Dcal}(\Q)$, have a negative impact on the Gibbs risk, as
\begin{align}
    r_{\Dcal}(\Q) =  e_{\Dcal}(\Q) + \tfrac{1}{2}d_{\Dcal}(\Q).
        \label{eq:risk-err-disa}
\end{align}
Despite this unfavorable behavior, many PAC-Bayesian results deal only with the Gibbs risks, thanks to a straightforward upper bound of the majority vote's risk which consists in upper-bounding it by twice the \mbox{Gibbs risk~\cite{LangfordShawe2002}, \ie,}
\begin{align}
    \label{eq:2gibbs} 
   r^{\MV}_{\Dcal}\!(\Q) \leq 2\, r_{\Dcal}(\Q) \ =\  2e_{\Dcal}(\Q) + d_{\Dcal}(\Q).
\end{align}
This bound is tight only when the Gibbs risk is low (\eg, when voters with large weights perform  well individually~\cite{GermainLacasseLavioletteMarchand2009,LangfordShawe2002}).
Recently, Masegosa {\it et al.}~\cite{MasegosaLorenzenIgelSeldin2020} propose to deal directly with the joint error as 
\begin{align}
    r^{\MV}_{\Dcal}(\Q) \le 4e_{\Dcal}(\Q) \ =\  2r_{\Dcal}(\Q) + 2e_{\Dcal}(\Q) - d_{\Dcal}(\Q). 
    \label{eq:masegosa}
\end{align}
Equation~\eqref{eq:masegosa} is tighter than Equation~\eqref{eq:2gibbs} if $e_{\Dcal}(\Q){\leq}\frac12 d_{\Dcal}(\Q)\Leftrightarrow r_{\Dcal}(\Q) {\leq} d_{\Dcal}(\Q)$; 
This captures the fact that the voters need to be sufficiently diverse and commit errors on different points.
However, when the joint error $e_{\Dcal}(\Q)$ exceeds~$\frac14$, the bound exceeds $1$ and is uninformative.
Another bound---known as the C-Bound in the PAC-Bayes literature~\cite{LacasseLavioletteMarchandGermainUsunier2006}---has been introduced to capture this trade-off between the Gibbs risk $r_{\Dcal}(\Q)$ and the disagreement $d_{\Dcal}(\Q)$, and is recalled in the following theorem.
\begin{theorem}[C-Bound] For any distribution $\Dcal$ on $\Xcal{\times}\Ycal$, for \mbox{any voters set}~$\Hcal$, for any distribution $\Q$ on~$\Hcal$, if $r_{\Dcal}(\Q){<}\tfrac{1}{2}{\iff}2e_{\Dcal}(\Q){+}d_{\Dcal}(\Q){<}1$, we have
\begin{align*}
    r^{\MV}_{\Dcal}(\Q)\  \le \ &1-\frac{\LP1-2r_{\Dcal}(\Q)\RP^2}{1-2d_{\Dcal}(\Q)}\quad \triangleq\quad  \CBound_{\Dcal}({\Q})\\
     &= 1-\frac{\big(1-\LB2e_{\Dcal}(\Q)+d_{\Dcal}(\Q)\RB\big)^2}{1-2d_{\Dcal}(\Q)}.
\end{align*}
The \textbf{empirical C-Bound} is denoted by $\CBound_{\Scal}({\Q})$ where the empirical disagreement is defined by $d_{\Scal}(\Q){=}\frac{1}{m}\sum_{i=1}^{m}\EE_{h_1\sim\Q}\EE_{h_2\sim\Q}\Ibf[h_1(\xbf_i){\ne}h_2(\xbf_i)],$ and the empirical joint error is defined by 
$e_{\Scal}(\Q){=} \frac{1}{m}\sum_{i=1}^m \EE_{h_1\sim\Q} \EE_{h_2\sim\Q} \Ibf[h_1(\xbf_i) {\ne} y_i]\Ibf[h_2(\xbf_i) {\ne} y_i]$.
\label{theorem:C-Bound}
\end{theorem}
As Equation~\eqref{eq:masegosa}, the C-Bound is tighter than Equation~\eqref{eq:2gibbs} when $r_{\Dcal}(\Q) \leq d_{\Dcal}(\Q)$ and looks for a good trade-off between individual risks and disagreement.
The main interest of the C-bound compared to Equation~\eqref{eq:masegosa} is that when $e_{\Dcal}(\Q)$ is close to $\tfrac{1}{4}$, the C-Bound can be close to $0$ depending on the value of the disagreement $d_{\Dcal}(\Q)$: the C-bound is then more precise.
Moreover, it is important to notice that the C-Bound is always tighter than Equation~\eqref{eq:masegosa} and tighter than Equation~\eqref{eq:2gibbs} when $r_{\Dcal}(\Q)\le d_{\Dcal}(\Q)$.
We summarize the relationships between Equations~\eqref{eq:2gibbs},~\eqref{eq:masegosa} and $\CBound_{\Dcal}({\Q})$ in the next theorem.
\begin{theorem}[From Germain {et al.}~\cite{RoyMarchandLaviolette2016} and Masegosa {et al.}~\cite{MasegosaLorenzenIgelSeldin2020}] For any distribution $\Dcal$ on $\Xcal\times\Ycal$, for any voters set $\Hcal$, for any distribution $\Q$ on $\Hcal$, if $r_{\Dcal}(\Q) < \tfrac{1}{2}$, we have
\begin{center}
\begin{minipage}{8.2cm}
\begin{enumerate}[\it (i)]
    \item $C_{\Dcal}(\Q) \le 4e_{\Dcal}(\Q) \le 2r_{\Dcal}(\Q)$,\quad if $r_{\Dcal}(\Q)\le d_{\Dcal}(\Q)$, 
    \item $2r_{\Dcal}(\Q) \le C_{\Dcal}(\Q) \le 4e_{\Dcal}(\Q)$,\quad otherwise.
\end{enumerate}
\end{minipage}
\end{center}
\end{theorem}

In this paper, we focus on the minimization of PAC-Bayesian generalization bounds on the C-Bound to get a low-risk majority vote.
In Section~\ref{section:pac-bayesian-C-Bounds}, we recall such PAC-Bayesian bounds that have been introduced in the literature.

\subsection{Related Works} 
Previous algorithms have been developed to minimize the \emph{empirical} \mbox{C-Bound} $\CBound_{\Scal}({\Q})$.
Roy~{\it et al.}~\cite{RoyLavioletteMarchand2011} first proposed \algor{MinCq} where this minimization is expressed as a quadratic problem.
\algor{MinCq} considers a specific voters' set to regularize the minimization process.
One  drawback of \algor{MinCq} is that the optimization problem is not scalable to large datasets.
Lately, Bauvin~{\it et al.}~\cite{BauvinCapponiRoyLaviolette2020} proposed \mbox{\algor{CB-Boost}} that minimizes $\CBound_{\Scal}({\Q})$ in a greedy procedure with the advantage to be more scalable while obtaining sparser majority vote. 
However, since both \algor{MinCq} and \algor{CB-Boost} minimize the empirical  $\CBound_{\Scal}({\Q})$, the PAC-Bayesian generalization bound associated with their learned majority vote predictors can be vacuous.
Note that \algor{CB-Boost} has been proposed to improve another algorithm called  \mbox{\algor{CqBoost}}~\cite{RoyMarchandLaviolette2016}.

\noindent When it comes to deriving a learning algorithm that directly minimizes a PAC-Bayesian bound, it is mentioned in the literature that optimizing a PAC-Bayesian bound on the C-bound is not trivial~\cite{MasegosaLorenzenIgelSeldin2020,LorenzenIgelSeldin2019}. 
This underlines the need of other majority vote learning algorithms based on the C-Bound, which motivates our contributions of Section~\ref{section:contribution}.

\section{PAC-Bayesian C-Bounds}
\label{section:pac-bayesian-C-Bounds}

We recall now three PAC-Bayesian generalization bounds on the \mbox{C-Bound} referred hereafter as  the {\bf PAC-Bayesian C-Bounds}. 
Considering these three approaches has the interest to offer a large coverage of the PAC-Bayesian \mbox{C-bound} literature.
Our contribution, described in Section~\ref{section:contribution}, consists  in deriving a self-bounding algorithm for each of these PAC-Bayesian C-Bounds.
This shows that the PAC-Bayesian C-Bound offers various ways to learn majority votes that might have been overlooked until now.

\subsection{An Intuitive Bound---McAllester's View}

We recall the most intuitive and interpretable PAC-Bayesian C-Bound~\cite{RoyMarchandLaviolette2016}.
It consists in upper-bounding separately the Gibbs risk $r_{\Dcal}(\Q)$ and the disagreement $d_{\Dcal}(\Q)$ with the usual PAC-Bayesian bound of McAllester~\cite{McAllester2003} that bounds the deviation between true and empirical values with the Euclidean distance. 

\begin{theorem}[PAC-Bayesian C-Bound of Roy {et al.}~\cite{RoyMarchandLaviolette2016}] 
\label{th:C-Bound-mcallester}
For any distribution $\D$ on $\Xcal{\times}\Ycal$, for any prior distribution $\P$ on $\Hcal$, for any $\delta{>}0$, we have
\allowdisplaybreaks[4]
\begin{align}
 &\PP_{\Scal\sim\Dcal^{m}}\Bigg( \forall \Q \mbox{ on }\Hcal,\   \CBound_{\D}({\Q}) \le \underbrace{1{-}\frac{\LP 1-2\min\LB\frac{1}{2},r_{\Scal}(\Q){+}\sqrt{\tfrac{1}{2}\psi_{r}(\Q)}\RB\RP^2}{1-2\max\LB 0, d_{\Scal}(\Q){-}\sqrt{\tfrac{1}{2}\psi_{d}(\Q)}\RB}}_{\displaystyle \CBound^{\tt M}_{\Scal}(\Q)} \Bigg) \geq 1{-}2\delta,\nonumber\\[-6mm]
 \label{eq:C-Bound-mcallester}
 \end{align}
 with $\psi_{r}(\Q) = \tfrac1m\!\left[ \KL(\Q\|\P){+}\ln\!\tfrac{2\sqrt{m}}{\delta}\right],\  \text{and}\ \psi_{d}(\Q) = \tfrac1m\!\left[ 2\,\KL(\Q\|\P){+}\ln\!\tfrac{2\sqrt{m}}{\delta}\right]$, and
     $\KL(\Q\|\P)=\EE_{h\sim\Q}\ln\frac{\Q(h)}{\P(h)}$ is the KL-divergence between $\Q$ and $\P$.
\label{theorem:C-Bound-mcallester}
\end{theorem}
While there is no algorithm that directly minimizes Equation~\eqref{eq:C-Bound-mcallester}, this kind of interpretable bound can be seen as a justification of the optimization of $r_{\Scal}(\Q)$ and $d_{\Scal}(\Q)$ in the empirical C-Bound such as for \algor{MinCq}~\cite{RoyLavioletteMarchand2011} or \mbox{\algor{CB-Boost}}~\cite{BauvinCapponiRoyLaviolette2020}.
In Section~\ref{section:contribution-mcallester}, we derive a first algorithm to directly minimize it.

However, this PAC-Bayesian C-Bound can have a severe disadvantage with a small $m$ and a Gibbs risk close \mbox{to $\tfrac{1}{2}$}: even for a $\KL(\Q\|\P)$ close to $0$ and a low empirical \mbox{C-Bound}, the value of the PAC-Bayesian C-Bound will be close to $1$.
To overcome this drawback, one solution is to follow another PAC-Bayesian point of view, the one proposed by Seeger~\cite{Seeger2002} that compares the true and empirical values through  \mbox{$\kl(a\|b) {=} a\log\LB\tfrac{a}{b}\RB {+} (1{-}a)\log\big[\tfrac{1{-}a}{1{-}b}\big]$}, knowing that \mbox{$|a{-}b| \le \sqrt{\frac12 \kl(a\|b)}$} (Pinsker's inequality).

In the next two subsections, we recall such bounds. 
The first one in Theorem~\ref{theorem:C-Bound-seeger} involves the risk and the disagreement, while the second one in Theorem~\ref{theorem:C-Bound-e-d} simultaneously bounds the joint error and the disagreement.

\subsection{A Tighter Bound---Seeger's view}
\label{section:C-Bound-seeger}

The PAC-Bayesian generalization bounds based on the Seeger's approach~\cite{Seeger2002} are known to produce tighter bounds \cite{GermainLacasseLavioletteMarchandRoy2015}.
As for Theorem~\ref{theorem:C-Bound-mcallester}, the result below bounds independently the Gibbs risk $r_{\Dcal}(\Q)$ and the disagreement $d_{\Dcal}(\Q)$.
\begin{theorem}[\scriptsize PAC-Bayesian C-Bound (PAC-Bound 1) of Germain et al.~\cite{GermainLacasseLavioletteMarchandRoy2015}] 
\label{theorem:C-Bound-seeger}
Under the same assumptions and notations as Theorem~\ref{theorem:C-Bound-mcallester}, we have 
\begin{align}
\label{eq:C-Bound-Seeger-original-bound}
&\PP_{\Scal{\sim}\Dcal^m}\Bigg(\forall \Q \mbox{ on } \Hcal,\  \CBound_{\D}({\Q}) \le \underbrace{1{-}\frac{\big(1{-}2\min\LB\frac{1}{2},  \klmax \LP r_{\Scal}(\Q) \;\middle|\; \psi_{r}(\Q)\RP\RB\big)^2}{1{-}2\max\LB 0,
\klmin 
\LP d_{\Scal}(\Q) \;\middle|\; \psi_{d}(\Q)\RP\RB}}_{\displaystyle\CBound^{\tt S}_{\Scal}(\Q)}\Bigg) \geq 1{-}2\delta,\nonumber\\[-6mm]
\end{align}
with $\klmax (q | \psi){=}\max\{ p\!\in\!(0,\! 1) | \kl(q\|p)\!\!\le\!\!\psi\}$, and $\klmin (q | \psi){=}\min\{ p\!\in\!(0,\! 1) | \kl(q\|p)\!\!\le\!\!\psi\}.$
\end{theorem}
The form of this bound makes the optimization a challenging task: the functions $\klmax $ and $\klmin $ do not benefit from closed-form solutions. 
However, we see in Section~\ref{section:C-Bound-seeger} that the optimization of $\klmax$ and $\klmin$ can be done by the bisection method~\cite{ReebDoerrGerwinnRakitsch2018}, leading to an easy-to-solve algorithm to optimize this PAC-Bayesian C-Bound.

\subsection{\mbox{Another Tighter Bound--Lacasse's View}}
\label{section:pac-bayesian-e-d}
The last theorem  on which we build our contributions is described below. 
Proposed initially by Lacasse {\it et al.}~\cite{LacasseLavioletteMarchandGermainUsunier2006}, its interest is that it simultaneously bounds the joint error and the disagreement (as explained by Germain {\it et al.}~\cite{GermainLacasseLavioletteMarchandRoy2015}). 
Here, to compute the bound, we need to find the worst C-Bound value that can be obtained with a couple of joint error and disagreement denoted by $(e,d)$ belonging to the set $A_\Scal(\Q)$ that is defined by
\begin{align*}
A_\Scal(\Q) = \Big\{ (e, d) \;\Big|\; \kl\LP e_{\Scal}(\Q),  d_{\Scal}(\Q)\|e, d\RP \le \kappa(\Q)  \Big\},
\end{align*}
where $\ \kappa(\Q) = \tfrac1m \left[ 2\KL(\Q\|\P) + \ln\!\tfrac{2\sqrt{m}+m}{\delta}\right]$,\\
and $\ \kl(q_1{,} q_2\| p_1{,} p_2) = q_1\ln\!\tfrac{q_1}{p_1} + q_2 \ln\!\tfrac{q_2}{p_2} + (1{-}q_1{-}q_2)\ln\!\tfrac{1{-}q_1{-}q_2}{1{-}p_1{-}p_2}$.\\
The set $A_\Scal(\Q)$ can actually contain some pairs not achievable by any $\Dcal$, it can then be restricted to the valid subset {\small $\widetilde{A}_\Scal(\Q)$} defined in the theorem below.
\begin{theorem}[\scriptsize PAC-Bayesian C-Bound (PAC-Bound 2) of Germain et al.~\cite{GermainLacasseLavioletteMarchandRoy2015}]
\label{theorem:C-Bound-e-d}
Under the same assumptions as Theorem~\ref{theorem:C-Bound-mcallester}, we have
\begin{align*}
    &\PP_{\Scal{\sim}\Dcal^m}\left(\forall \Q \mbox{ on }\Hcal,\  \CBound_{\D}({\Q})\le \sup_{(e, d)\in \widetilde{A}_\Scal(\Q)} \LB 1- \frac{\LP 1-(2e+d)\RP^2}{1-2d} \RB \right) \geq 1{-}\delta,\\
    &\text{where }
    \widetilde{A}_\Scal(\Q) = \left\{ (e, d)\!\in\! A_\Scal(\Q) \ \Big|\  d\le 2 \sqrt{e}{-}2e\,,\ 2e{+}d<1 \right\}.
\end{align*}
\end{theorem}
Optimizing this bound {\it w.r.t.} $\Q$ can be challenging, since it boils down to optimize indirectly the set {\small $\widetilde{A}_\Scal(\Q)$}.
Hence, a direct optimization by gradient descent is not possible. 
In Section~\ref{section:contribution-e-d} we derive an approximation easier to optimize. 

\section{Self-Bounding Algorithms for PAC-Bayesian C-Bounds} 
\label{section:contribution}

In this section, we present our contribution that consists in proposing three self-bounding algorithms to directly minimize the PAC-Bayesian C-Bounds.

\subsection{Algorithm Based on McAllester's View}
\label{section:contribution-mcallester}
\begin{algorithm}[t]
  \caption{Minimization of Equation \eqref{eq:C-Bound-mcallester} by GD}
  \begin{algorithmic}
    \State{{\bf Given: } learning sample $\Scal$,  prior distribution $\P$ on $\Hcal$, the objective function $G_{\Scal}^{\tt M}(\Q)$\\
    \phantom{\bf Given: } Update function\footnotemark\  \textsc{update-$\Q$}}
    \State{{\bf Hyperparameters: } 
    number of iterations $T$}
    \Function{minimize-$\Q$}{}
    \State{$\Q \leftarrow \P$}
    \State{\textbf{for} $t\leftarrow 1$ to $T$ {\bf do} $\Q\leftarrow$\textsc{update-$\Q$}$(G^{\tt M}_{\Scal}(\Q))$}
    \State{\Return{$\Q$}}
    \EndFunction
  \end{algorithmic}
  \label{algo:mcallester}
\end{algorithm}
\footnotetext{\textsc{update-$\Q$} is a generic update function, \ie, it can be for example a standard update of GD or the update of another algorithm like Adam~\cite{KingmaBa2015} or COCOB~\cite{OrabonaTommasi2017}.} 

We derive in Algorithm~\ref{algo:mcallester} a method to directly minimize the  PAC-Bayesian \mbox{C-Bound} of Theorem~\ref{theorem:C-Bound-mcallester} by Gradient Descent (GD).
An important aspect of the optimization is that if $r_{\Scal}(\Q) {+} \text{\footnotesize$\sqrt{\!\tfrac{1}{2}\psi_{r}(\Q)}$}\ge\tfrac{1}{2}$, the gradient of the numerator in $\CBound^{\tt M}_{\Scal}(\Q)$ with respect to $\Q$ is 0 which makes the optimization impossible.
Hence, we aim at minimizing the following constraint optimization problem:
\begin{align*}
    \min_\Q& \underbrace{\LB 1-\frac{\LP 1-2\min\LB\frac{1}{2}, \!r_{\Scal}(\Q) {+}\sqrt{\!\tfrac{1}{2}\psi_{r}(\Q)}\RB\RP^2}{1-2\max\LB 0, d_{\Scal}(\Q){-}\sqrt{\tfrac{1}{2}\psi_{d}(\Q)}\RB}\RB}_{\displaystyle \CBound^{\tt M}_{\Scal}(\Q)}\quad\text{s.t }\quad r_{\Scal}(\Q) {+}\sqrt{\!\tfrac{1}{2}\psi_{r}(\Q)}\le\tfrac{1}{2}.
\end{align*}
From this formulation, we deduce a non-constrained optimization problem:$\ $ $\min_\Q\left[ \CBound^{\tt M}_{\Scal}(\Q) + \Bbf(r_{\Scal}(\Q) {+} \text{\scriptsize$\sqrt{\!\tfrac{1}{2}\psi_{r}(\Q)}$}{-}\tfrac{1}{2})\right]$, where $\Bbf$ is the barrier function defined as $\Bbf(a)\!=\!0$ if $a\!\le\! 0$ and $\Bbf(a)\!=\!+\infty$ otherwise.
Due to the nature of $\Bbf$, this problem is not suitable for optimization: the objective function will be infinite when $a\!>\!0$.
To tackle this drawback, we replace $\Bbf$ by the approximation introduced by Kervadec {\it et al.}~\cite{KervadecDolzYuanDesrosiersGrangerAyed2019} called the log-barrier extension and defined as 
\begin{align*}
    \Bbf_{\lambda}(a) = \left\{\begin{array}{cc}
        -\tfrac{1}{\lambda}\ln(-a), & \text{if } a \le -\tfrac{1}{\lambda^2}, \\[1mm]
        \lambda a{-}\tfrac{1}{\lambda}\ln(\tfrac{1}{\lambda^2}){+}\tfrac{1}{\lambda}, & \text{otherwise.}
    \end{array}\right.
\end{align*}
In fact, $\Bbf_{\lambda}$ tends to $\Bbf$ when $\lambda$ tends to $+\infty$.
Compared to the standard log-barrier\footnote{The reader can refer to \cite{BoydVandenberghe2014} for an introduction of interior-point methods.}, the function $\Bbf_{\lambda}$ is differentiable even when the constraint is not satisfied, \ie, when $a > 0$.
By taking into account the constraint \mbox{$r_{\Scal}(\Q) {+}\sqrt{\tfrac{1}{2}\psi_{r}(\Q)}\!\le\!\tfrac{1}{2}$}, we solve by GD with Algorithm~\ref{algo:mcallester} the following problem:
\begin{align*}
   \min_{\Q}\ G_{\Scal}^{\tt M}(\Q) \ =\ \min_{\Q}\  \CBound^{\tt M}_{\Scal}(\Q) + \Bbf_{\lambda}\left(r_{\Scal}(\Q) {+} \sqrt{\tfrac{1}{2}\psi_{r}(\Q)} {-}\tfrac{1}{2}\right).
\end{align*}
For a given $\lambda$, the optimizer will thus find a solution with a good trade-off between minimizing $\CBound^{\tt M}_{\Scal}(\Q)$ and the log-barrier extension function $\Bbf_{\lambda}$.
As we show in the experiments, minimizing the McAllester-based bound does not lead to the tightest bound.
Indeed, as mentioned in Section \ref{section:pac-bayesian-C-Bounds}, such bound is looser than Seeger-based bounds, and leads to a looser PAC-Bayesian C-Bound. 

\subsection{Algorithm Based on Seeger's View}
\label{section:contribution-seeger}

In order to obtain better generalization guarantees, we should optimize the Seeger-based C-bound of Theorem~\ref{theorem:C-Bound-seeger}. 
In the same way as in the previous section, we seek at minimizing the following optimization problem:
\begin{align*}
    \min_\Q& \underbrace{\left[1{-}\frac{\big(1{-}2\min\LB\frac{1}{2},  \klmax \LP r_{\Scal}(\Q) \;\middle|\; \psi_{r}(\Q)\RP\RB\big)^2}{1{-}2\max\LB 0,
  \klmin 
  \LP d_{\Scal}(\Q) \;\middle|\; \psi_{d}(\Q)\RP\RB}\right]}_{\displaystyle\CBound^{\tt S}_{\Scal}(\Q)}\quad\text{s.t }\quad \klmax \LP r_{\Scal}(\Q) \;\middle|\; \psi_{r}(\Q)\RP\le\tfrac{1}{2},
\end{align*}
with $\klmax (q | \psi){=}\max\{ p\!\in\!(0,\! 1) | \kl(q\|p)\!\!\le\!\!\psi\}$, and $\klmin (q | \psi){=}\min\{ p\!\in\!(0,\! 1) | \kl(q\|p)\!\!\le\!\!\psi\}.$
For the same reasons as for deriving Algorithm~\ref{algo:mcallester}, we propose to solve by GD:
\begin{align*}
 \min_{\Q} \  G^{\tt S}_{\Scal}(\Q) \ =\ \min_{\Q}\  \CBound^{\tt S}_{\Scal}(\Q) + \Bbf_{\lambda}\big(\klmax \LP r_{\Scal}(\Q) \;\middle|\; \psi_{r}(\Q)\RP{-}\tfrac{1}{2}\big).
\end{align*}
The main challenge to optimize it is to evaluate $\klmax$ or $\klmin$ and to compute their derivatives.
To do so, we follow the bisection method to calculate $\klmax$ and $\klmin$ proposed by Reeb {\it et al.}~\cite{ReebDoerrGerwinnRakitsch2018}.
This method is summarized in the functions {\sc compute-$\klmax(q|\psi)$} and {\sc compute-$\klmin(q|\psi)$} of Algorithm~\ref{algo:seeger}, and consists in refining iteratively an interval $[p_{\text{min}}, p_{\text{max}}]$ with $p\in [p_{\text{min}}, p_{\text{max}}]$ such that \mbox{$\kl(q\|p)\!=\!\psi$}.
For the sake of completeness, we provide the derivatives of $\klmin$ and $\klmax$ with respect to $q$ and $\psi$, that are:
\begin{align}
    \frac{\partial {\rm k(}q|\psi)}{\partial q} = \frac{\ln\frac{1-q}{1-{\rm k(}q|\psi)}-\ln\frac{q}{{\rm k(}q|\psi)}}{\frac{1-q}{1-{\rm k(}q|\psi)}-\frac{q}{{\rm k(}q|\psi)}}, \text{ and } \frac{\partial {\rm k(}q|\psi)}{\partial \psi} = \frac{1}{\frac{1-q}{1-{\rm k(}q|\psi)}-\frac{q}{{\rm k(}q|\psi)}},\label{eq:deriv-kl}
\end{align}
with ${\rm k}$ is either $\klmin$ or $\klmax$.
To compute the derivatives with respect to the posterior $\Q$, we use the chain rule for differentiation with a deep learning framework (such as PyTorch~\cite{Paszke2019}). 
The global algorithm is summarized in Algorithm~\ref{algo:seeger}.

\begin{algorithm}[h!]
 \caption{Minimization of Equation \eqref{eq:C-Bound-mcallester} by GD}
  \begin{algorithmic}
  \State{{\bf Given: } learning sample $\Scal$,  prior distribution $\P$ on $\Hcal$, the objective function $G_{\Scal}^{\tt M}(\Q)$\\
  \phantom{\bf Given: } Update function \textsc{update-$\Q$}}
    \State{{\bf Hyperparameters: } 
    number of iterations $T$}
    \Function{minimize-$\Q$}{}
    \State{$\Q \leftarrow \P$}
    \For{$t\leftarrow 1$ to $T$}
        \State{Compute $G^{\tt S}_{\Scal}(\Q)$ using {\sc compute-$\klmax(q|\psi)$} and {\sc compute-$\klmin(q|\psi)$}}
        \State{$\Q\leftarrow$\textsc{update-$\Q$}$(G^{\tt S}_{\Scal}(\Q))$ (thanks to the derivatives in Equation~\eqref{eq:deriv-kl})}
    \EndFor
    \State{\Return{$\Q$}}
    \EndFunction\\
    {\centerline{\rule{0.75\linewidth}{0.25pt}}}
    \State{{\bf Hyperparameters: } tolerance $\epsilon$, maximal number of iterations $T_\text{max}$}
    \Function{compute-$\klmax(q|\psi)$ (resp. compute-$\klmin(q|\psi)$)}{}
    \State{$p_{\text{max}}{\leftarrow}1$ and $p_{\text{min}}{\leftarrow}q$ (resp. $p_{\text{max}}{\leftarrow}q$ and $p_{\text{min}}{\leftarrow}0$)}
    \For{$t\leftarrow 1$ to $T_{\text{max}}$}
        \State{$p = \tfrac{1}{2}\LB p_{\text{min}}{+}p_{\text{max}}\RB$}
        \State{\textbf{if} $\kl(q\|p)=\psi$ or  $(p_{\text{min}}{-}p_{\text{max}})<\epsilon$ \textbf{ then return} $p$}
        \State{\textbf{if} $\kl(q\|p) > \psi$ \textbf{ then } $p_\text{max}=p$ (resp. $p_\text{min}=p$)}
        \State{\textbf{if} $\kl(q\|p) < \psi$ \textbf{ then } $p_\text{min}=p$ (resp. $p_\text{max}=p$)}
    \EndFor
     \State{\Return{$p$}}
    \EndFunction
  \end{algorithmic}
  \label{algo:seeger}
\end{algorithm}

\subsection{Algorithm Based on Lacasse's View}
\label{section:contribution-e-d}
Theorem~\ref{theorem:C-Bound-e-d} jointly upper-bounds the joint error $e_{\Dcal}(\Q)$ and the disagreement $d_{\Dcal}(\Q)$; But as pointed out in Section~\ref{section:pac-bayesian-e-d} its optimization can be hard.
To ease its manipulation, we derive below a C-Bound resulting of a reformulation of the constraints involved in the set \mbox{{\small $\widetilde{A}$}$_{\Scal}(\Q)\!=\!\{ (e, d)\!\in\! A_\Scal(\Q) \,|\, d\!\le\! 2\sqrt{e}{-}2e, 2e{+}d\!<\!1 \}$}. 
\begin{theorem}
Under the same assumptions as Theorem~\ref{theorem:C-Bound-mcallester}, we have
\begin{align}
    &\PP_{\Scal{\sim}\Dcal^m} \Bigg( \CBound_{\D}({\Q})\le \sup_{(e, d)\in \widehat{A}_\Scal(\Q)} \underbrace{\LB 1- \frac{\big[ 1-(2e+d)\big]^2}{1-2d} \RB}_{\displaystyle\CBound^{\tt L}(e,d)}\Bigg)\geq 1{-}\delta,\label{eq:new-C-Bound-e-d}\\ 
 \nonumber    &\text{where}\ \ \widehat{A}_\Scal(\Q) = \LC (e, d)\!\in\! A_\Scal(\Q) \,\middle|\, d\le 2\sqrt{\min\LP e, \tfrac{1}{4}\RP}{-}2e,\  d<\tfrac{1}{2} \RC,\\
 \nonumber \text{and}\; A_\Scal(\Q) &{=} \big\{ \!(e{,} d) \big| \kl\!\LP e_{\Scal}(\Q){,}  d_{\Scal}(\Q)\|e{,} d\RP\! \le\! \kappa(\Q)\!  \big\}, \text{with}\; \textstyle \kappa(\Q) {=} 
   \frac{2\KL(\Q\|\P){+}\ln\!\tfrac{2\sqrt{m}{+}m}{\delta}}{m}\!.
   \end{align}
\label{theorem:new-C-Bound-e-d}
\end{theorem}
\begin{proof}
Beforehand, we explain how we fixed the constraints involved in $\widehat{A}_\Scal(\Q)$.
We add to  $A_\Scal(\Q)$ three constraints: $d\!\le\!  2\sqrt{e}{-}2e$ (from Prop. 9 of~\cite{GermainLacasseLavioletteMarchandRoy2015}), $d\!\le\!  1{-}2e$, and $d\!<\!\tfrac{1}{2}$. 
We remark that when $e\!\le\! \tfrac{1}{4}$, we have $2\sqrt{e}{-}2e\!\le\! 1{-}2e$.
Then, we merge $d\!\le\!  2\sqrt{e}{-}2e$ and $d\!\le\!  1{-}2e$ into $d\!\le\!  2\sqrt{\min\LP e, \tfrac{1}{4}\RP}{-}2e$.
Indeed, we have
$$d \le  2\sqrt{\min(e, \tfrac{1}{4})}{-}2e \ \Longleftrightarrow\ \left\{
\begin{array}{ll}
d \le  2\sqrt{e}-2e&\text{ if }e\le \frac{1}{4},\\[1mm]
d < 1{-}2e &\text{ if }e\geq\frac14.
\end{array}\right.$$
We prove now that under the constraints involved in $\widehat{A}_\Scal(\Q)$, we still have a valid bound on $\CBound_{\D}({\Q})$.
To do so, we consider two cases.\\[1mm]
{\bf Case 1}: If for all $(e,d)\in \widehat{A}_\Scal(\Q)$ we have $2e{+}d\!<\!1$.\\ 
In this case $(e_{\Dcal}(\Q), d_{\Dcal}(\Q))\!\in\! \widehat{A}_\Scal(\Q)$, then we have $2e_{\Dcal}(\Q){+}d_{\Dcal}(\Q)\!<\!1$ and 
Theorem~\ref{theorem:C-Bound} holds. 
We have
$\CBound_{\Dcal}(\Q) = 
1-\frac{\LB 1-\LP 2e_{\Dcal}(\Q)+d_{\Dcal}(\Q)\RP \RB^2}{1-2d_{\Dcal}(\Q)}
\leq \sup_{(e,d)\in\widehat{A}_{\Scal}(\Q)}\CBound^{\tt L}(e,d)$.\\[1mm]
{\bf Case 2}: If there exists $(e,d)\in \widehat{A}_\Scal(\Q)$  such that $2e{+}d\!=\!1$.\\ 
We have $\sup_{(e,d)\in\widehat{A}_{\Scal}(\Q)}\CBound^{\tt L}(e,d) = 1$ that is a valid bound on $\CBound_{\D}({\Q})$. \qed
\end{proof}

\noindent Theorem~\ref{theorem:new-C-Bound-e-d} suggests then the following constrained optimization problem: 
\begin{align*}
    &\min_{\Q}\!\left\{\! \sup_{\scalebox{0.8}{$(e{,}d){\in}\!\LB0,\!\tfrac{1}{2}\RB^2$}}\!\!
    \LP 1\!-\! \frac{\big[ 1\!-\!(2e\!+\!d)\big]^2}{1\!-\!2d}\! \RP\,   \text{s.t.}\,  (e, d)\! \in\! \widehat{A}_\Scal(\Q) \! \right\} \!\text{ s.t. } 2e_{\Scal}(\Q){+}d_{\Scal}(\Q)\!\le\! 1,
\end{align*}
with $\widehat{A}_\Scal(\Q) {=} \big\{ (e, d) \big| d\le 2\sqrt{\min\LP e, \tfrac{1}{4}\RP}{-}2e,\  d\!<\!\tfrac{1}{2},\ \kl\!\LP e_{\Scal}(\Q){,}  d_{\Scal}(\Q)\|e{,} d\RP\! \le\! \kappa(\Q)\!  \big\}$.
Actually, we can rewrite this constrained optimization problem into an unconstrained one using the barrier function. 
We obtain
\begin{align}
    \min_{\Q} \Bigg\{&\max_{\scalebox{0.8}{$(e{,}d){\in}\!\LB0,\!\tfrac{1}{2}\RB^2$}}  \Bigg(
     \CBound^{\tt L}(e,d) - \Bbf\!\LB d {-} 2\sqrt{\min\LP e, \tfrac{1}{4}\RP}{-}2e\RB -\Bbf\!\LB d{-}\tfrac{1}{2}\RB \nonumber\\
     &\quad - \Bbf\Big[\kl\LP e_{\Scal}(\Q),  d_{\Scal}(\Q)\|e, d\RP{-}\kappa(\Q)\Big]
   \Bigg) + \Bbf\Big[ 2e_{\Scal}(\Q){+}d_{\Scal}(\Q){-}1\Big]\Bigg\}, \label{eq:k_Q_param_nu}
\end{align}
where $\CBound^{\tt L}(e,d)=1-\tfrac{\LP 1-(2e+d)\RP^2}{1-2d}$ if $d\!<\!\frac12$, and $\CBound^{\tt L}(e,d)\!=\!1$ otherwise.
However, this problem cannot be optimized directly by GD.
In this case, we have a \mbox{min-max} optimization problem, \ie, for each descent step we need to find the couple $(e, d)$ that maximizes the $\CBound^{\tt L}(e,d)$ given the three constraints that define $\widehat{A}_\Scal(\Q)$ before updating the posterior distribution $\Q$.

First, to derive our optimization procedure, we focus on the inner maximization problem when $e_{\Scal}(\Q)$ and $d_{\Scal}(\Q)$ are fixed in order to find the optimal $(e,d)$.
However, the function $\CBound^{\tt L}(e,d)$ we aim at maximizing is not concave for all $(e,d)\!\in\!\R^2$, implying that the  implementation of its maximization can be hard\footnote{For example, when using \algor{CVXPY}~\cite{DiamondBoyd2016}, that uses Disciplined Convex Programming (DCP~\cite{GrantBoydYe2006}), the maximization of a non-concave function is not possible.}. 
Fortunately, $\CBound^{\tt L}(e,d)$ is quasi-concave~\cite{GermainLacasseLavioletteMarchandRoy2015} for $(e, d)\in[0, 1]\times[0, \frac{1}{2}]$.
Then by  definition of quasi-concavity, we have:
\begin{align*}
&\forall \alpha\in [0,1],\quad \left\{ (e, d) \,\middle|\, 1- \frac{\big[ 1-(2e+d)\big]^2}{1-2d} \ge 1-\alpha \right\}\\
\Longleftrightarrow\quad  &\forall \alpha\in [0,1],\quad  \left\{ (e, d) \ \middle|\  
\alpha(1{-}2d)-\Big[1{-}(2e{+}d)\Big]^2 \ge 0\right\}.
\end{align*}

\begin{algorithm}[t]
  \caption{Minimization of Equation~\eqref{eq:new-C-Bound-e-d} by GD}
  \begin{algorithmic}
    \State{{\bf Given: } learning sample $\Scal$, prior $\P$ on $\Hcal$, the objective function $G^{e^*\!{,}d^*}_{\Scal}\!\!(\Q)$\\
    \phantom{\bf Given: } Update function \textsc{update-$\Q$}}
    \State{{\bf Hyperparameters: }
    number of iterations $T$}
    \Function{minimize-$\Q$}{}
    \State{$\Q \leftarrow \P$}
    \For{$t\leftarrow 1$ to $T$}
        \State{$(e^*, d^*) \leftarrow $\Call{maximize-$e$-$d$}{$e_{\Scal}(\Q), d_{\Scal}(\Q)$}}
        \State{$\Q \leftarrow$ \textsc{update-$\Q$}$(G^{e^*\!{,}d^*}_{\Scal}\!\!(\Q))$}
    \EndFor
    \State{\Return{$\Q$}}
    \EndFunction\\
    {\centerline{\rule{0.75\linewidth}{0.25pt}}}
      \State{{\bf Given:} learning sample $\Scal$, joint error $e_{\Scal}(\Q)$, disagreement $d_{\Scal}(\Q)$}
\State{{\bf Hyperparameters: }tolerance $\epsilon$}
\Function{maximize-$e$-$d$}{$e_{\Scal}(\Q), d_{\Scal}(\Q)$}
\State{$\alpha_{\text{min}} = 0$ and $\alpha_{\text{max}}=1$}
\While{$\alpha_{\text{max}}-\alpha_{\text{min}}>\epsilon$}
\State{$\alpha=\tfrac{1}{2}(\alpha_{\text{min}}+\alpha_{\text{max}})$}
\State{$(e, d) \leftarrow$ Solve Equation~\eqref{eq:min-e-d}}
\State{\textbf{if} $\CBound^{\tt L}(e,d) \ge 1{-}\alpha$ \textbf{then} $\alpha_{\text{max}} \leftarrow \alpha$ \textbf{else} $\alpha_{\text{min}} \leftarrow \alpha$}
\EndWhile
\State{\Return{$(e, d)$}}
\EndFunction
\end{algorithmic}
\label{algo:e-d-min}
\end{algorithm}
\noindent Hence, for any fixed $\alpha\!\in\![0, 1]$ we can look for $(e, d)$ that maximizes $\CBound^{\tt L}(e,d)$ and respects the constraints involved in $\widehat{A}_\Scal(\Q)$.
This is equivalent to solve the following problem for a given $\alpha\in[0, 1]$:
\begin{align}
    \max_{(e, d)\in[0, \frac{1}{2}]^2}\quad  \alpha(1{-}2d)-\Big[1{-}(2e{+}d)\Big]^2\label{eq:min-e-d}\hspace{2.5cm}\\
    \text{ s.t. } \quad d\le 2\sqrt{\min\LP e, \tfrac{1}{4}\RP}{-}2e\quad\text{ and }\quad\kl\LP e_{\Scal}(\Q),  d_{\Scal}(\Q)\|e, d\RP\le \kappa(\Q)\nonumber.
\end{align}
In fact, we aim at finding $\alpha\in[0,1]$ such that the maximization of Equation~\eqref{eq:min-e-d} leads to $1{-}\alpha$ equal to the largest value of $C^{\tt L}(e, d)$ under the constraints.
To do so, we make use of the ``Bisection method for quasi-convex optimization''~\cite{BoydVandenberghe2014} that is summarized in  \textsc{maximize-$e$-$d$} in Algorithm~\ref{algo:e-d-min}.
We denote by $(e^{*\!}, d^*)$ the solution of Equation~\eqref{eq:min-e-d}.
It remains then to solve the outer minimization problem that becomes:
\begin{align*}
    \min_{\Q}\big\{\ \Bbf\LB 2e_{\Scal}(\Q){+}d_{\Scal}(\Q){-}1\RB  -\Bbf\LB\kl\LP e_{\Scal}(\Q),  d_{\Scal}(\Q)\|e^{*\!}, d^*\RP{-}\kappa(\Q)\RB\ \big\}.
\end{align*}
Since the barrier function $\Bbf$ is not suitable for optimization, we approximate this problem by replacing $\Bbf$ by the log-barrier extension $\Bbf_\lambda$, \ie, we have
\begin{align*}
   \min_{\Q}\  G^{e^*\!{,}d^*}_{\Scal}\!\!(\Q)\  =\   \min_{\Q}\big\{\ & \Bbf_{\lambda}\LB 2e_{\Scal}(\Q){+}d_{\Scal}(\Q){-}1\RB
   \\
    &-\Bbf_\lambda \LB\kl\LP e_{\Scal}(\Q),d_{\Scal}(\Q)\|e^{*\!},d^*\RP{-}\kappa(\Q)\RB\   \big\}.
\end{align*}
The global method is summarized in Algorithm~\ref{algo:e-d-min}.
\noindent As a side note, we  mention that the classic Danskin Theorem~\cite{Danskin1966} used in min-max optimization theory is not applicable in our case since our objective function is not differentiable for all $(e, d)\in[0, \tfrac{1}{2}]^2$. 
We discuss this point in Supplemental.

\section{Experimental Evaluation} 
\label{section:expe}

\begin{table}[t]
\centering
\caption{Comparison of the true risks ``$r^{\text{MV}}_{\Tcal}$'' and bound values ``Bnd'' obtained for each algorithm.
``Bnd'' is the value of the bound that is optimized, excepted for {\sc MinCq} and {\sc CB-Boost} for which we report the bound obtained with Theorem~\ref{theorem:new-C-Bound-e-d} instantiated with the majority vote learned. 
Results in {\bf bold} are the couple ($r^{\text{MV}}_{\Tcal}\!$,Bnd) associated to \textbf{the lowest risk} value.
{\it Italic} and \underline{underlined} results are the couple  ($r^{\text{MV}}_{\Tcal}\!$,Bnd) associated respectively to \textit{the lowest bound} value and \underline{the second lowest bound} values.
\label{table:expe}
}
\scalebox{0.82}{
\setlength{\tabcolsep}{1.25mm}
\begin{tabular}{@{\hspace{-2mm}}c@{} |c c| cc |cc || cc |cc || cc | cc }
\bottomrule[1pt]
& \multicolumn{2}{c|}{\bf\texttt{Alg.\ref{algo:mcallester}}} & \multicolumn{2}{c|}{\bf\texttt{Alg.\ref{algo:seeger}}} & \multicolumn{2}{c||}{\bf\texttt{Alg.\ref{algo:e-d-min}}}
& \multicolumn{2}{c|}{\scriptsize \sc CB-Boost}
& \multicolumn{2}{c||}{\footnotesize\sc MinCq}
& \multicolumn{2}{c|}{\scriptsize\sc Masegosa}
& \multicolumn{2}{c}{\algoGibbs}
\\[1mm]
& $r^{\text{MV}}_{\Tcal}$ & Bnd & $r^{\text{MV}}_{\Tcal} $ & Bnd & $r^{\text{MV}}_{\Tcal} $ & Bnd & $r^{\text{MV}}_{\Tcal} $ & Bnd & $r^{\text{MV}}_{\Tcal} $ & Bnd & $r^{\text{MV}}_{\Tcal} $ & Bnd & $r^{\text{MV}}_{\Tcal} $ & Bnd\\
\toprule[1pt]
\text{{\scriptsize letter:}AvsB}    &       .009 &      .323 &       .018 &      .114 &       \textbf{.000} &      \textbf{.085} &          {.000} &        {.104} &       .009 &      .451 &          \underline{.004} &         \underline{.070} &    \textit{.018} &   \textit{.056} \\
\text{{\scriptsize letter:}DvsO}    &       \textbf{.013} &      \textbf{.469} &       .018 &      .298 &       .018 &      .205 &          .022 &         .224 &       .022 &      .999 &          \underline{.018} &         \underline{.185} &    \textit{.044} &   \textit{.174}\\
\text{{\scriptsize letter:}OvsQ}    &       .017 &      .489 &       .017 &      .332 &       \textbf{.009} &      \textbf{.229} &          .017 &         .249 &       .039 &      1 &          \underline{.013} &         \underline{.210} &    \textit{.030} &   \textit{.201} \\
\text{credit}        &       .141 &      .912 &       .141 &      .874 &       \underline{.129} &      \underline{.816} &          .144 &         .855 &       \textbf{.126} &      \textbf{.929} &          .132 &         .869 &    \textit{.150} &   \textit{.651} \\
\text{glass}         &       .047 &      .904 &       .047 &      .832 &       \underline{.056} &      \underline{.798} &          \textbf{.037 }&         \textbf{.911} &       .056 &      .999 &          .056 &         .903 &    \textit{.047} &   \textit{.566} \\
\text{heart}         &       .250 &      .976 &       .264 &      .962 &       \underline{.250} &      \underline{.955} &          .270 &         .981 &       .270 &      1 &          \textbf{.243} &         \textbf{1.19} &    \textit{.250} &   \textit{.787} \\
\text{tictactoe}     &       .063 &      .815 &       .084 &      .750 &       \textbf{.056} &      \textbf{.610} &          .063 &         .649 &       .071 &     .782 &          \underline{.058} &         \underline{.580} &    \textit{.152} &   \textit{.511} \\
\text{usvotes}       &       .041 &      .741 &       .046 &      .584 &       .037 &      .508 &          .037 &         .590 &       .046 &      .985 &          \underline{\textbf{.032}} &         \underline{\textbf{.490}} &    \textit{.060} &   \textit{.342} \\
\text{wdbc}           &       .060 &      .725 &       .053 &      .603 &       .032 &      .523 &          \textbf{.025} &         \textbf{.591} &       .039 &      .992 &           \underline{.035} &          \underline{.513} &   \textit{ .063} &   \textit{.362} \\
\text{{\scriptsize mnist:}1vs7}     &       .006 &      .161 &       .005 &      .061 &        \underline{\textbf{.005}} &      \underline{\textbf{.038}} &          .005 &         .040 &       .015 &      .994 &          \textit{.006} &         \textit{.034} &    .006 &   .043 \\
\text{{\scriptsize mnist:}4vs9}     &       .017 &      .238 &       .016 &      .167 &       \underline{.016} &       \underline{.110} &          .016 &         .113 &       .046 &      .960 &         \textit{\textbf{.016}} &         \textit{\textbf{.106}} &    .063 &   .148 \\
\text{{\scriptsize mnist:}5vs6}     &       .011 &      .210 &       .011 &      .124 &        \underline{\textit{.011}} &       \underline{\textit{.078}} &          .011 &         .081 &       .035 &      .999 &          \textbf{.011} &         \textbf{.073} &    .036 &   .109 \\
\text{{\scriptsize fash:}COvsSH}\, &       \textbf{.108} &      \textbf{.462} &       .109 &      .433 &        \underline{.110} &       \underline{.366} &          .110 &         .371 &       .185 &      .894 &          \textit{.111} &         \textit{.358} &    .146 &   .409 \\
\text{{\scriptsize fash:}SAvsBO}\, &       .018 &      .217 &       .018 &      .134 &        \underline{.019} &       \underline{.094} &          .019 &         .097 &       .034 &      1 &          \textit{\textbf{.018}} &        \textit{\textbf{.087}} &    .020 &   .114 \\
\text{{\scriptsize fash:}TOvsPU}\, &       .029 &      .245 &       .029 &      .165 &       \textbf{.029} &      \textbf{.133} &          .030 &         .136 &       .045 &      .809 &           \underline{.030} &          \underline{.125} &    \textit{.051} &   \textit{.123} \\
\text{adult}         &       .163 &      .532 &       .163 &      .514 &       \underline{\textbf{.163}} &      \underline{\textbf{.492}} &          .163 &         .495 &       .204 &      1 &           \underline{\textbf{.163}} &          \underline{\textbf{.492}} &    \textit{.200} &   \textit{.413} \\
\midrule[1pt]\midrule[1pt]
{\tt Mean } & .062 & .526 & .065 & .434 & .059 & .378 & .061 & .405 & .078 & .925 & .059 & .393 & .083 & .313\\
\bottomrule[1pt]
\end{tabular}
}
\end{table}

\begin{figure}[t]
    \centering
    \begin{tikzpicture}[thick,scale=0.48, every node/.style={scale=0.48}]    
        \node[] at (0,0) {\includegraphics[scale=0.5]{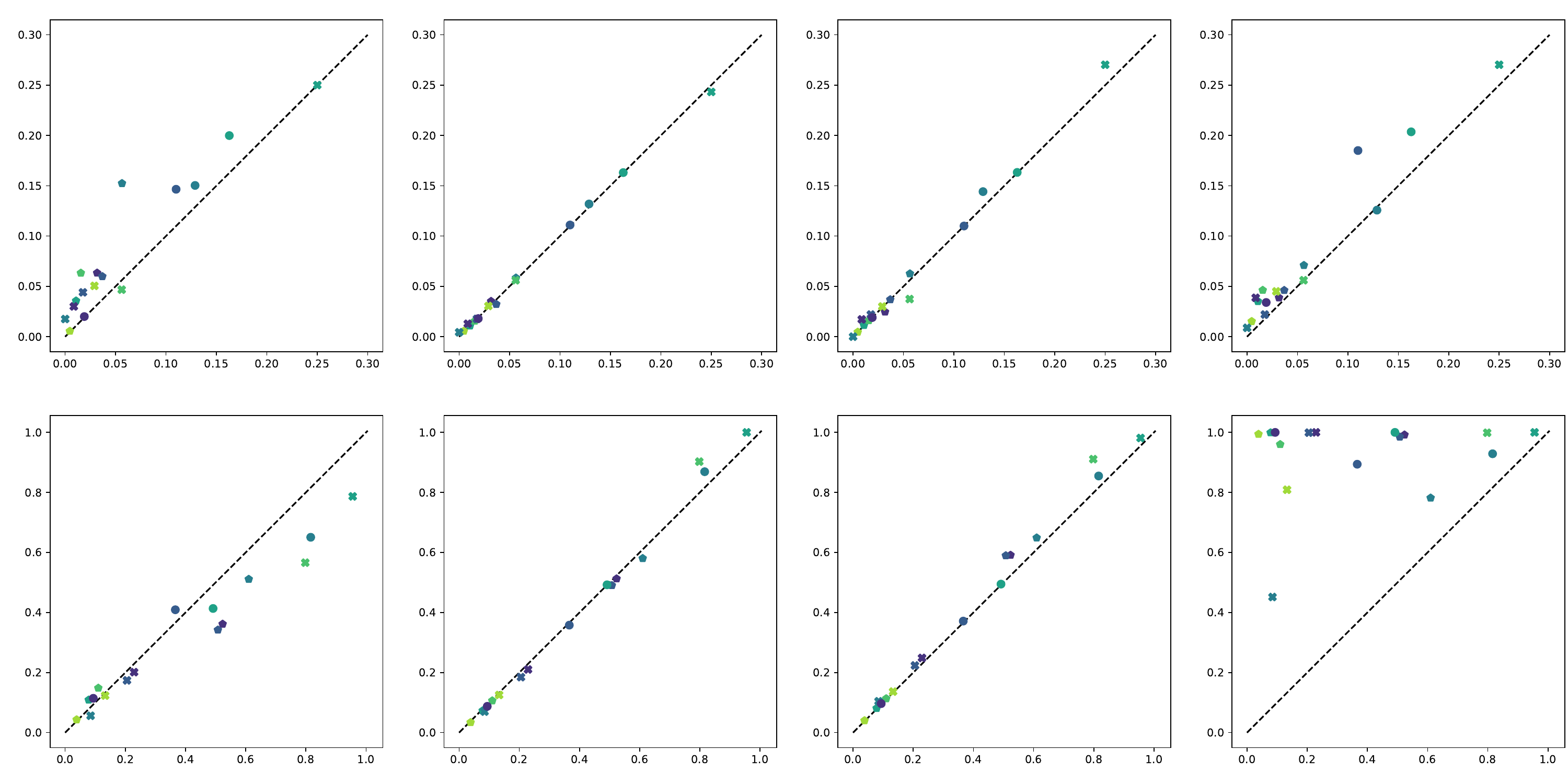}};
        \node[] at (0,7) {\Large \bf TEST RISKS COMPARISON};
        \node[] at (-9.2, 6.4) {\Large {\bf\texttt{Alg.\ref{algo:e-d-min}}} vs. \algoGibbsLarge};
        \node[] at (-2.8, 6.4) {\Large {\bf\texttt{Alg.\ref{algo:e-d-min}}} vs. {\sc Masegosa}};
        \node[] at (3.5 , 6.4) {\Large {\bf\texttt{Alg.\ref{algo:e-d-min}}} vs. {\sc CB-Boost}};
        \node[] at (10, 6.4) {\Large {\bf\texttt{Alg.\ref{algo:e-d-min}}} vs. {\sc MinCq}};
        \node[] at (0,0) {\Large \bf BOUND VALUES COMPARISON};
        \node[] at (-9.2, -6.4) {\Large {\bf\texttt{Alg.\ref{algo:e-d-min}}} vs. \algoGibbsLarge};
        \node[] at (-2.8, -6.4) {\Large {\bf\texttt{Alg.\ref{algo:e-d-min}}} vs. {\sc Masegosa}};
        \node[] at (3.5, -6.4) {\Large {\bf\texttt{Alg.\ref{algo:e-d-min}}} vs. {\sc CB-Boost}};
        \node[] at (10, -6.4) {\Large {\bf\texttt{Alg.\ref{algo:e-d-min}}} vs. {\sc MinCq}};
    \end{tikzpicture}
    \caption{Pairwise comparisons of the test risks (first line) and the bounds (second line) between Algorithm~\ref{algo:e-d-min} and the baseline algorithms.
    Algorithm~\ref{algo:e-d-min} is represented on the x-axis, while the y-axis is used for the other approaches. 
    Each dataset corresponds to a point in the plot and a point above the diagonal indicates that Algorithm~\ref{algo:e-d-min} is better.}
    \label{fig:bounds}
\end{figure}

\subsection{Empirical Setting} 
Our experiments\footnote{Experiments are done with PyTorch~\cite{Paszke2019} and CVXPY~\cite{DiamondBoyd2016}. The source code is \mbox{available} at \url{https://github.com/paulviallard/ECML21-PB-CBound}.} have a two-fold objective: \textit{(i)} assessing the guarantees given by the associated PAC-Bayesian bounds, and \textit{(ii)} comparing the performance of the different C-bound based algorithms in terms of risk optimization.
To achieve this objective, we compare the three algorithms proposed in this paper to the following state-of-the-art PAC-Bayesian methods for majority vote learning:
\begin{itemize}
    \item[\textbullet]  \algor{MinCq} \cite{RoyLavioletteMarchand2011} and \algor{CB-Boost} \cite{BauvinCapponiRoyLaviolette2020} that are based on the minimization of the empirical \mbox{C-Bound}. 
For comparison purposes and since {\sc MinCq} and {\sc CB-Boost} do not explicitly minimize a PAC-Bayesian bound, we report the bound values of Theorem~\ref{theorem:new-C-Bound-e-d} instantiated with the models learned;
\item[\textbullet] The algorithm proposed by Masegosa {\it et al.}~\cite{MasegosaLorenzenIgelSeldin2020} that optimizes a PAC-Bayesian bound on  $r^{\MV}_{\Dcal}(\Q) \le 4e_{\Dcal}(\Q)$ (see Theorem~9 of ~\cite{MasegosaLorenzenIgelSeldin2020});
\item[\textbullet] An algorithm\footnote{The algorithm \algoGibbs~is similar to Algorithm~\ref{algo:seeger}, but without the numerator of the C-Bound (\ie, the disagreement). More details are given in the Supplemental.}, denoted by \algoGibbs, to optimize a PAC-Bayesian bound based only on the Gibbs risk~\cite{LangfordShawe2002}: 
$r^{\text{MV}}_\Dcal(\Q) \le 2r_\Dcal(\Q)  \le 2\klmax(r_{\Scal}(\Q)|\psi_r(\Q)).$ 
\end{itemize}
We follow a general setting similar to the one of Masegosa {\it et al.}~\cite{MasegosaLorenzenIgelSeldin2020}.
The prior distribution $\P$ on $\Hcal$ is set as the uniform distribution, and the voters in $\Hcal$ are decision trees: $100$ trees are learned with $50\%$ of the training data (the remaining part serves to learn the posterior $\Q$).
More precisely, for each tree $\sqrt{d}$ features of the $d$-dimensional input space are selected, and the trees are learned by using the Gini criterion until the leaves are pure. 

\noindent In this experiment, we consider $16$ classic datasets\footnote{An overview of the datasets is presented in the Supplemental.} that we split into a train set~$\Scal$ and a test set $\Tcal$.
We report for each algorithm in Table~\ref{table:expe}, the test risks (on $\Tcal$) and the bound values (on $\Scal$, such that the bounds hold with prob. at least $95\%$).
The parameters of the algorithms are selected as follows.
\mbox{\bf\footnotesize 1) For}  Masegosa's algorithm we kept the default parameters~\cite{MasegosaLorenzenIgelSeldin2020}.
\mbox{\bf\footnotesize 2) For} all the other bounds minimization algorithms, we set $T{=}2,000$ iterations for all the datasets except for {adult}, {fash} and {mnist} where $T{=}200$.
We fix the objective functions with $\lambda{=}100$, and we use COCOB-Backprop optimizer~\cite{OrabonaTommasi2017} as {\sc update-$\Q$} (its parameter remains the default one).
For Algorithm~\ref{algo:e-d-min}, we fix the tolerance $\epsilon{=}.01$, {\it resp.} $\epsilon{=}10^{-9}$, to compute $\klmin$, {\it resp.} $\klmax$.
Furthermore, the maximal number of iterations $T_{\text{max}}$ in {\sc maximize-$e$-$d$} is set to $1,000$.
\mbox{\bf\footnotesize 3) For}  {\tt \algor{MinCq}}, we select the margin parameter among $20$ values uniformly distributed in $[0,\tfrac{1}{2}]$ by 3-fold cross validation. 
Since this algorithm is not scalable due to its high time complexity, we reduce the training set size to $m{{=}}400$ when learning with {\tt \algor{MinCq}} on the large datasets: {adult}, {fash} and {mnist} ({\tt \algor{MinCq}} is still competitive with less data on this datasets).
For {\tt \algor{CB-Bound}} which is based on a Boosting approach, we fix the maximal number of boosting iterations to $200$.

\subsection{Analysis of the Results}
Beforehand, we compare only our three self-bounding algorithms.
From Table~\ref{table:expe}, as expected we observe that Algorithm~\ref{algo:mcallester} based on the McAllester's bound (that is more interpretable but less tight) provides the worst bound.
Algorithm~\ref{algo:e-d-min} always provides tighter bounds than Algorithms~\ref{algo:mcallester} and~\ref{algo:seeger}, and except for {\scriptsize letter:}DvsO, {\scriptsize fash:}COvsSH, and  {\scriptsize fash:}SAvsBO Algorithm~\ref{algo:e-d-min} leads to the lowest test risks.
We believe that Algorithm~\ref{algo:e-d-min} based on the Lacasse's bound provides lower bounds than Algorithm~\ref{algo:seeger} based on the Seeger's bound because the Lacasse's approach bounds simultaneously the joint error and the disagreement.
Algorithm~\ref{algo:e-d-min} appears then to be the best algorithm among our three self-bounding algorithms that minimize a PAC-Bayesian C-Bound.\\

In the following we focus then on comparing our best contribution represented by Algorithm~\ref{algo:e-d-min} to the baselines; Figure~\ref{fig:bounds} summarizes this comparison.\\
\indent First, \algoGibbs~gives the lowest bounds among all the algorithms, but at the price of the largest risks.
This clearly illustrates the limitation of considering \textit{only} the Gibbs risk as an estimator of the majority vote risk: As discussed in Section~\ref{sec:gibbs}, the Gibbs risk is an unfair estimator since an increase of the diversity between the voters can have a negative impact on the Gibbs risk.\\
\indent Second, compared to Masegosa's approach, the results are comparable:
Algorithm~\ref{algo:e-d-min} tends to provide tighter bounds, and similar performances that lie in the same order of magnitude, as illustrated in Table~\ref{table:expe}. 
This behavior was expected since minimizing the bound of Masegosa~\cite{MasegosaLorenzenIgelSeldin2020} or the PAC-Bayesian C-Bound boils down to minimize a trade-off between the risk and the disagreement. \\
\indent Third, compared to empirical C-bound minimization algorithms, we see that Algorithm~\ref{algo:e-d-min} outputs better results than \algor{CB-Boost} and \algor{MinCq} for which the difference is significative and the bounds are close to $1$ (\ie, non-informative).
Optimizing the risk bounds tend then to provide better guarantees that justify that optimizing the empirical C-bound is often too optimistic.\\

\indent Overall, from these experiments, our Algorithm~\ref{algo:e-d-min} is the one that provides the best trade-off between having good performances in terms of risk optimization  and ensuring good theoretical guarantees with informative bounds.

\section{Conclusion and Future Work}
\label{section:conclusion}
In this paper, we present new learning algorithms driven by the minimization of PAC-Bayesian generalization bounds based on the C-Bound. 
More precisely, we propose to solve three optimization problems, each one derived from an existing PAC-Bayesian bound.
Our methods belong to the class of \emph{self-bounding} learning algorithms: The learned predictor comes with a tight and statistically valid risk upper bound. 
Our experimental evaluation has confirmed the quality of the learned predictor and the tightness of the bounds with respect to state-of-the-art methods minimizing the C-Bound.

As future work, we would like to study extensions of this work to provide meaningful bounds for learning (deep) neural networks. In particular, an interesting perspective would be to adapt the C-Bound to control the diversity and the weights in a neural network.

\section*{Acknowledgements}
This work was supported by the French Project {\sc apriori} {\small ANR-18-CE23-0015}.
Moreover, Pascal Germain is supported by the NSERC Discovery grant RGPIN-2020-07223 and the Canada CIFAR AI Chair Program.
The authors thank R\'emi Emonet for insightful discussions.

\newpage

\appendix
\begin{center}\large\bf
SUPPLEMENTAL OF\\[2mm]
Self-Bounding Majority Vote Learning Algorithms\\by the Direct Minimization\\of a Tight PAC-Bayesian C-Bound
\end{center}

\section{Section~\ref{section:expe}---Details on the Datasets}
\label{appendix:dataset}
Table~\ref{table:expe-overview} presents an overview of the datasets we use in our experiments (the split train/test, the dimensionality and the url to the dataset).

\begin{table}[th!]
\centering
\caption{Datasets overview.}
\scalebox{0.70}{
\begin{tabular}{c |c|c|c|c }
\bottomrule[1pt]

& $|\Scal|$ & $|\Tcal|$ & Dim. & Link\\
\toprule[1pt]
\texttt{letter:OvsQ} & 1303 & 233 & 16 &
\url{https://archive.ics.uci.edu/ml/datasets/letter+recognition}\\
\texttt{letter:DvsO} & 1331 & 227 & 16 &
\url{https://archive.ics.uci.edu/ml/datasets/letter+recognition}\\
\texttt{letter:AvsB} & 1327 & 228 & 16 &
\url{https://archive.ics.uci.edu/ml/datasets/letter+recognition}\\
\texttt{credit} & 327 & 326 & 46 &
\url{https://archive.ics.uci.edu/ml/datasets/Credit+Approval}\\
\texttt{heart} &  149 & 148 & 13 &
\url{https://archive.ics.uci.edu/ml/datasets/heart+disease}\\
\texttt{glass} &  107 & 107 & 9 &
\url{https://archive.ics.uci.edu/ml/datasets/glass+identification}\\
\texttt{tictactoe} & 479 & 479 & 9 &
\url{https://archive.ics.uci.edu/ml/datasets/Tic-Tac-Toe+Endgame}\\
\texttt{usvotes} & 218 & 217 & 48 &
\url{https://archive.ics.uci.edu/ml/datasets/congressional+voting+records}\\
\texttt{wdbc} & 285 & 284 & 30 &
\url{https://archive.ics.uci.edu/ml/datasets/Breast+Cancer+Wisconsin+(Diagnostic)}\\
\midrule[1pt]\midrule[1pt]
\texttt{adult} & 30162 & 15060 & 104 &
\url{https://archive.ics.uci.edu/ml/datasets/adult}\\
\texttt{mnist:1vs7} &  13007 & 2163 & 784 &
\url{http://yann.lecun.com/exdb/mnist}\\
\texttt{mnist:4vs9} &  11791 & 1991 & 784 &
\url{http://yann.lecun.com/exdb/mnist}\\
\texttt{mnist:5vs6} & 11339 & 1850 & 784 &
\url{http://yann.lecun.com/exdb/mnist}\\
\texttt{fash:TOvsPU} & 12000 & 2000 & 784 &
\url{https://github.com/zalandoresearch/fashion-mnist}\\
\texttt{fash:SAvsBO}  & 12000 & 2000 & 784 &
\url{https://github.com/zalandoresearch/fashion-mnist}\\
\texttt{fash:COvsSH} & 12000 & 2000 & 784 &
\url{https://github.com/zalandoresearch/fashion-mnist}\\
\bottomrule[1pt]
\end{tabular}
}
\label{table:expe-overview}
\end{table}

\section{Section~\ref{section:contribution-e-d}---About Danskin's Theorem}
\label{appendix:danskin}

As mentioned in the main paper, in the context of the justification of the function {\sc maximize-$e$-$d$} in Algorithm~\ref{algo:e-d-min}, we now discuss the possible  application of Danskin's Theorem \cite[Section I]{Danskin1966}. 
The statement of the theorem is as follows.

\begin{theorem}[Danskin's Theorem]
Let $\Acal\subset \R^m$ be a compact set and $\phi:\R^n\times \Acal\rightarrow \R$ s.t. for all $\abf\in\Acal$, we have that $\phi$ is continuously differentiable, then $\Phi(\xbf) = \max_{\abf\in\Acal}\phi(\xbf, \abf)$ is directionally differentiable with  directional derivatives
\begin{align*}
    \Phi'(\xbf, \dbf) = \max_{\abf\in\Acal^*}\LA\dbf, \nabla_{\xbf}\phi(\xbf, \abf)\RA,
\end{align*}
where $\Acal^*=\LC \abf^* \;\middle|\; \phi(\xbf, \abf^*)=\max_{\abf\in\Acal}\phi(\xbf, \abf)\RC$ and $\LA\cdot,\cdot\RA$ is the dot product.
\label{theorem:danskin}
\end{theorem}

\noindent To optimize a problem $\min_{\xbf\in\R^n}\Phi(\xbf)$ with $\Phi(\xbf) = \max_{\abf\in\Acal}\phi(\xbf, \abf)$, this theorem tells us that under several assumptions, if we know a maximizer $\abf\in\Acal$, then, we have an analytical expression of the directional derivatives of $\Phi(\xbf)$. 
Thus, from this theorem, we also know a gradient to minimize the problem $\min_{\xbf\in\R^n}\Phi(\xbf)$. 

\begin{corollary}[Madry {et al.}~\cite{MadryMakelovSchmidtTsiprasVladu2018}] 
\label{cor:madry} Assuming that the conditions of Theorem~\ref{theorem:danskin} are fulfilled and let $\abf^*\in\Acal^{*}$ be a maximizer of $\phi$. 
If $\dbf=\nabla_{\xbf}\phi(\xbf, \abf^*)$ with $\LN\dbf\RN^2_2>0$ then $-\dbf$ is a descent direction for $\Phi(\xbf)$, \ie, $\Phi'(\xbf, \dbf) > 0$.
\end{corollary}
\begin{proof} By definition of the directional derivative $\Phi'(\xbf, \dbf)$ and the direction $\dbf$, we have:
\begin{align*}
        \Phi'(\xbf, \dbf) &= \max_{\abf\in\Acal^*}\LA\dbf,\nabla_{\xbf}\phi(\xbf, \abf)\RA\\
        &= \max_{\abf\in\Acal^*}\LA\nabla_{\xbf}\phi(\xbf, \abf^*),\nabla_{\xbf}\phi(\xbf, \abf)\RA \ge \LN\nabla_{\xbf}\phi(\xbf, \abf^*)\RN_{2}^{2}> 0.
    \end{align*}
    \end{proof}
\noindent Then, for each iteration of the min/max problem optimization, we can {\it (i)} optimize the inner maximization problem, {\it (ii)} fix the maximizer $\abf^*\in\Acal$ and apply a gradient descent step with the derivative $\nabla_{\xbf}\phi(\xbf, \abf^*)$.
However, as we mentioned in the main paper, the assumptions are not fulfilled in our case to apply Theorem~\ref{theorem:danskin} since our inner objective in Equation~\eqref{eq:k_Q_param_nu} or its approximation
\centerline{$\displaystyle\CBound^{\tt L}(e,d){-}\Bbf\!\LB d {-} 2\sqrt{\min\LP e, \tfrac{1}{4}\RP}{-}2e\RB\!{-}\Bbf\!\LB d{-}\tfrac{1}{2}\RB\!{-}\Bbf_{\lambda}\Big[\kl\LP e_{\Scal}(\Q),  d_{\Scal}(\Q)\|e, d\RP{-}\kappa(\Q)\Big]$}
is not differentiable everywhere in the compact set $[0, \tfrac{1}{2}]^2$.
However, we never encounter problematic cases and this strategy is thus valid for optimizing our proposed approximation.
In practice, we have found that it is indeed an efficient and sound strategy.

\section{Section~\ref{section:expe}---About Optimizing $2\klmax \LP r_{\Scal}(\Q)| \psi_{r}(\Q)\RP$}
\label{appendix:algo-2-gibbs}
To minimize the bound $2(\klmax \LP r_{\Scal}(\Q) \;\middle|\; \psi_{r}(\Q)\RP)$, we adopt the algorithm (denoted as \algoGibbs\ in the setting description of the experiments of Section~\ref{section:expe}) similar to Algorithm~\ref{algo:seeger}.
Indeed, we use instead the objective function:
\begin{align}
    \min_{\Q} 2(\klmax \LP r_{\Scal}(\Q) \;\middle|\; \psi_{r}(\Q)\RP).\label{eq:min-2r}
\end{align}
The algorithm is described in Algorithm~\ref{algo:2r} below. 
\begin{algorithm}[h!]
 \caption{Minimization of Equation \eqref{eq:min-2r} by GD}
  \begin{algorithmic}
  \State{{\bf Given: } learning sample $\Scal$,  prior distribution $\P$ on $\Hcal$, update function \textsc{update-$\Q$}}
    \State{{\bf Hyperparameters: }
    number of iterations $T$}
    \Function{minimize-$\Q$}{}
    \State{$\Q \leftarrow \P$}
    \For{$t\leftarrow 1$ to $T$}
        \State{Compute $\klmax \LP r_{\Scal}(\Q) \;\middle|\; \psi_{r}(\Q)\RP$ using {\sc compute-$\klmax(q|\psi)$}}.
        \State{$\Q\leftarrow$\textsc{update-$\Q$}$(\klmax \LP r_{\Scal}(\Q) \;\middle|\; \psi_{r}(\Q)\RP)$ (thanks to  Equation~\eqref{eq:deriv-kl})}
    \EndFor\\
    \Return{$\Q$}
    \EndFunction
  \end{algorithmic}
  \label{algo:2r}
\end{algorithm}

\section{Section~\ref{section:expe}---Additional Experiments}
\label{appendix:expe}
We report in Figure~\ref{figure:expe2} and Figure~\ref{figure:expe3}, the empirical joint error and disagreement obtained on the different datasets.
As for Table~\ref{table:expe} and Figure~\ref{fig:bounds}, this illustrates that the solutions found by \texttt{Alg.3}, {\sc Masegosa} and {\sc CB-Boost} are similar while {\sc MinCq} and \algoGibbs~provide very different solutions.

\begin{figure}[H]
\centering
\begin{tikzpicture}[thick,scale=0.43, every node/.style={scale=0.43}]
\node[inner sep=0pt] at (-2.0,0.0) {\includegraphics[scale=1.0]{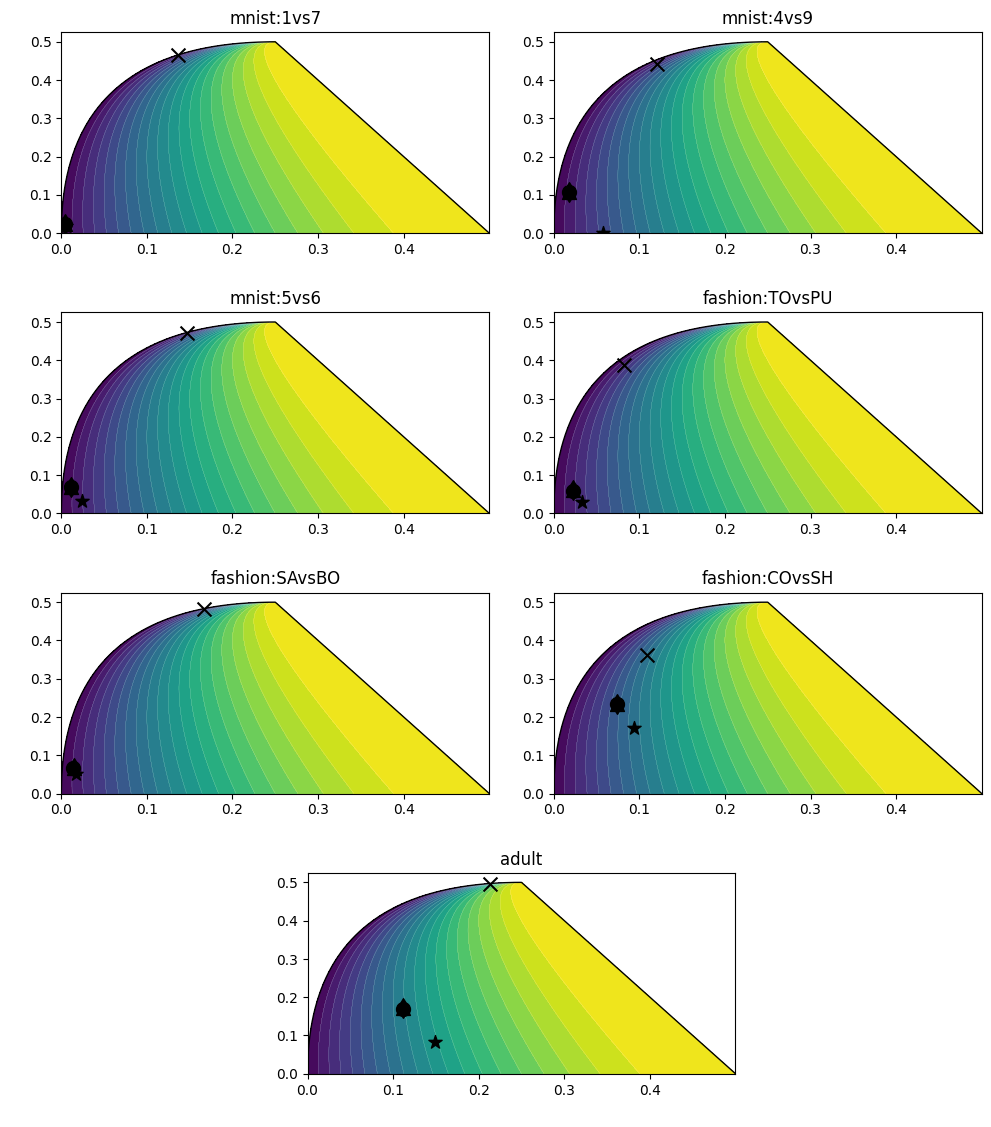}};
\node[inner sep=0pt] at (-2.0,-17.0) {\includegraphics[scale=1.0]{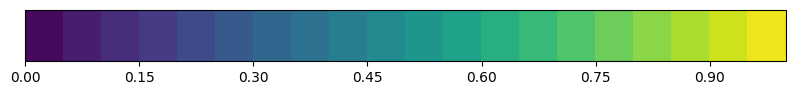}};
\node[] at (9.2,-16.7) {\LARGE$\boldsymbol{C_{\Scal}(\Q)}$};
\end{tikzpicture}
\caption{Representation of all the possible values of the empirical C-Bound $C_{\Scal}(\Q)$ in function of the disagreement $d_{\Scal}(\Q)$ (y-axis) and joint error $e_{\Scal}(\Q)$ (x-axis).
We report the values obtained on different datasets by \texttt{Alg.3} ($\blacklozenge$), {\sc Masegosa} ($\blacktriangle$), \algoGibbs~($\bigstar$), \texttt{CB-Boost} ($\bullet$), and \mbox{\texttt{MinCq} ($\boldsymbol{\times}$)}.}
\label{figure:expe2}
\end{figure}

\begin{figure}[H]
\centering
\begin{tikzpicture}[thick,scale=0.5, every node/.style={scale=0.5}]
\node[inner sep=0pt] at (-2.0,0.0) {\includegraphics[scale=1.0]{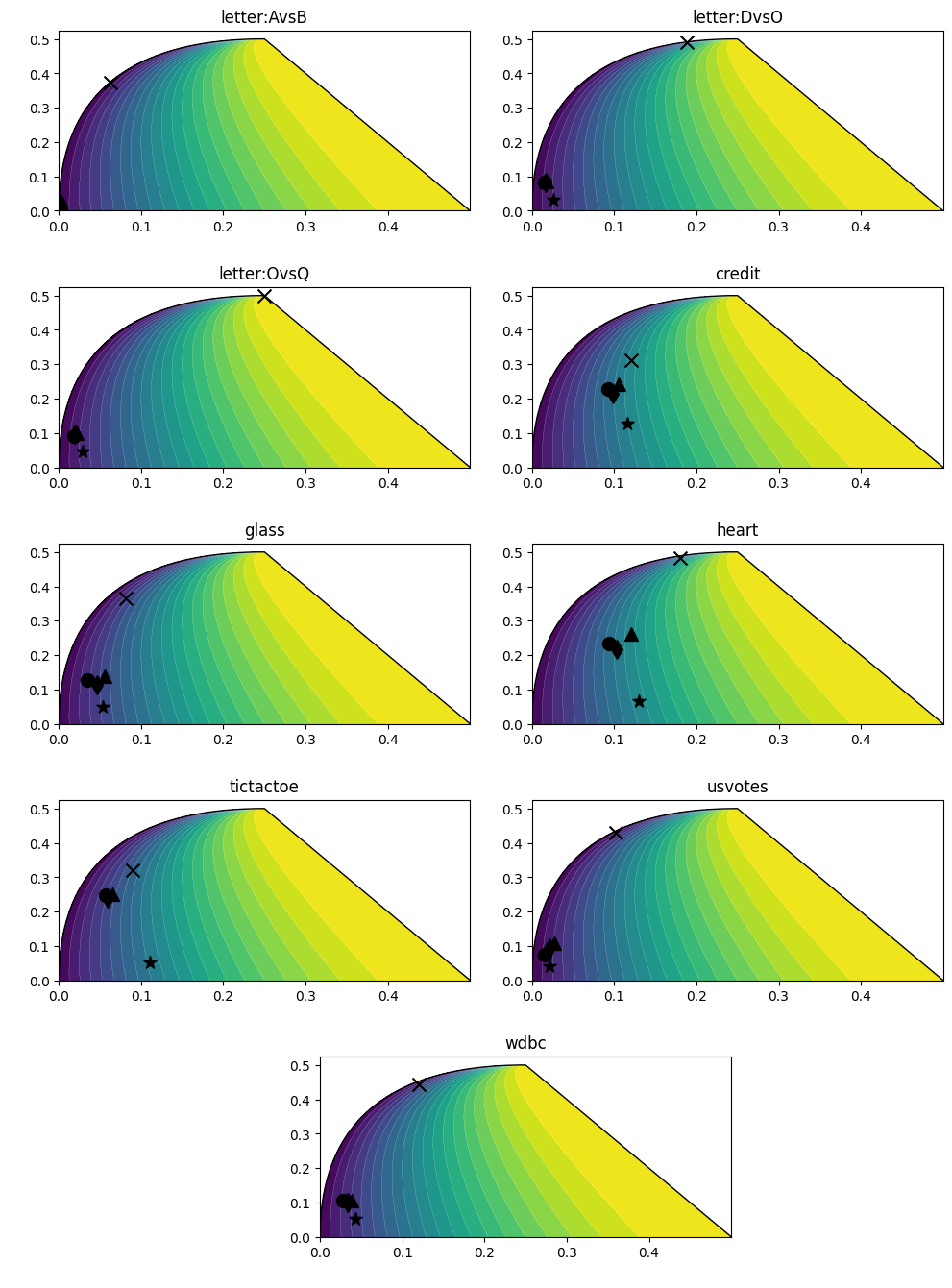}};
\node[inner sep=0pt] at (-2.0,-19.0) {\includegraphics[scale=1.0]{cbound_legend.png}};
\node[] at (9.2,-18.7) {\LARGE$\boldsymbol{C_{\Scal}(\Q)}$};
\end{tikzpicture}
\caption{Representation of all the possible values of the empirical C-Bound $C_{\Scal}(\Q)$ in function of the disagreement $d_{\Scal}(\Q)$ (y-axis) and joint error $e_{\Scal}(\Q)$ (x-axis).
We report the values obtained on different datasets by \texttt{Alg.3} ($\blacklozenge$), {\sc Masegosa} ($\blacktriangle$), \algoGibbs~($\bigstar$), \texttt{CB-Boost} ($\bullet$), and \mbox{\texttt{MinCq} ($\boldsymbol{\times}$)}.}
\label{figure:expe3}
\end{figure}
\end{document}

%% file: notations.tex
\newcommand{\abf}{\ensuremath{\mathbf{a}}}
\newcommand{\Bbf}{\ensuremath{\mathbf{B}}}
\newcommand{\dbf}{\ensuremath{\mathbf{d}}}

\newcommand{\Ibf}{\ensuremath{\mathbf{I}}}
\newcommand{\xbf}{\ensuremath{\mathbf{x}}}

\newcommand{\Acal}{\ensuremath{\mathcal{A}}}
\newcommand{\Dcal}{\ensuremath{\mathcal{D}}}
\newcommand{\Hcal}{\ensuremath{\mathcal{H}}}
\newcommand{\Pcal}{\ensuremath{\mathcal{P}}}
\newcommand{\Qcal}{\ensuremath{\mathcal{Q}}}
\newcommand{\Scal}{\ensuremath{\mathcal{S}}}
\newcommand{\Tcal}{\ensuremath{\mathcal{T}}}
\newcommand{\Xcal}{\ensuremath{\mathcal{X}}}
\newcommand{\Ycal}{\ensuremath{\mathcal{Y}}}

\newcommand{\Rbb}{\ensuremath{\mathbb{R}}}

\newcommand{\RA}{\right\rangle}
\newcommand{\LA}{\left\langle}
\newcommand{\LB}{\left[}
\newcommand{\RB}{\right]}
\newcommand{\LC}{\left\{}
\newcommand{\RC}{\right\}}
\newcommand{\RN}{\right\|}
\newcommand{\LN}{\left\|}
\newcommand{\LP}{\left(}
\newcommand{\RP}{\right)}

\newcommand{\ie}{{\em i.e.\/}}
\newcommand{\eg}{{\em e.g.\/}}

\DeclareMathOperator*{\EE}{\mathbb{E}}
\DeclareMathOperator*{\PP}{\mathrm{Pr}}

\newcommand{\sign}{\operatorname{sign}}
\newcommand{\KL}{{\rm KL}}
\newcommand{\kl}{{\rm kl}}
\newcommand{\klmin}{\underline{\kl}}
\newcommand{\klmax}{\overline{\kl}}

\newcommand{\MV}{\text{MV}}
\newcommand{\MVQ}{\MV_{\!\Q}}
\newcommand{\CBound}{{C}}

\newcommand{\algor}[1]{{\footnotesize\sc #1}}
\newcommand{\algorLarge}[1]{{\Large\sc #1}}
\newcommand{\algoGibbs}{\algor{2r}}
\newcommand{\algoGibbsLarge}{\algorLarge{2r}}

\newcommand{\D}{\Dcal}
\renewcommand{\P}{\Pcal}
\newcommand{\Q}{\Qcal}
\newcommand{\R}{\Rbb}